\documentclass{article}

\usepackage{amsmath}
\usepackage{amsthm}
\usepackage{amsfonts}
\usepackage{amssymb}

\usepackage{comment}
\usepackage{caption}
\usepackage{subcaption}

\usepackage{mathtools}
\DeclarePairedDelimiter{\ceil}{\lceil}{\rceil}
\DeclarePairedDelimiter{\SetBrack}{\{}{\}}

\usepackage{algorithm}
\usepackage{algorithmic}

\usepackage{tikz}
\usetikzlibrary{arrows}

\newtheorem{thm}{Theorem}[section]
\newtheorem{cor}[thm]{Corollary}
\newtheorem{lem}[thm]{Lemma}
\newtheorem{claim}[thm]{Claim}

\newcommand{\NN}{N}

\newcommand{\oneover}[1]{\frac{1}{#1}}
\newcommand{\spaceo}{\hspace{2 mm}}
\newcommand{\setsep}{ \spaceo | \spaceo}
\newcommand{\half}{\frac{1}{2}}

\newcommand{\Prob}[1]{\mathbb{P}\left( #1 \right)}
\newcommand{\Exp}{\mathbb{E}}

\newcommand{\exponent}[1]{exp\left( #1 \right)}
\newcommand{\eps}{\epsilon}
\newcommand{\Abs}[1]{\left| #1 \right|}

\newcommand{\argmax}{\operatornamewithlimits{argmax}}
\newcommand{\quater}{\frac{1}{4}}

\newcommand{\AppSecRestarts}{A}

\newcommand{\AppSecLFR}{C}
\newcommand{\AppSecLFROverlap}{D}

\title{Community Detection via Measure Space Embedding}

\author{
  Mark Kozdoba \\
  \texttt{markk@tx.technion.ac.il}
 \and
 Shie Mannor\\
  \texttt{shie@ee.technion.ac.il}
}

\title{Overlapping Communities Detection \\via Measure Space Embedding}
\date{}

\begin{document}

\maketitle

\begin{abstract} 
We present a new algorithm for community detection. The algorithm uses random walks to embed the graph in a 
space of measures, after which a modification of $k$-means in that space is applied. The algorithm is 
therefore fast and easily parallelizable. We evaluate the algorithm on standard random graph benchmarks, 
including some overlapping community benchmarks,  and find its performance to be better or at least as good as 
previously known  algorithms. We also prove a linear time (in number of edges) guarantee for the 
algorithm on a $p,q$-stochastic block model with 
 where $p \geq c\cdot N^{-\half +  \epsilon}$ and $p-q \geq c' \sqrt{p N^{-\half +  \epsilon} \log N}$. 
\end{abstract} 


\section{Introduction}
Community detection in graphs, also known as graph clustering, is a problem where one wishes to identify subsets of the vertices of a graph such that the connectivity inside the subset is in some way denser than the connectivity of the subset with the rest of the graph. Such subsets are referred to as communities, and it often happens in applications that if two vertices belong to the same community, they have  similar application-related qualities. This in turn may allow for a higher level analysis of the graph, in terms of communities instead of individual nodes. Community detection finds applications in a diversity of fields, such as social networks analysis, communication and traffic design, in biological networks, and, generally, in most fields where meaningful graphs can arise (see, for instance,  \cite{Fortunato201075} for a survey).
In addition to direct applications to graphs, community detection can, for instance, be also applied to general Euclidean space clustering problems, by transforming the metric to a weighted graph structure (see  \cite{vL} for a survey). 

Community detection problems come in different flavours, depending on whether the graph in question is simple, or weighted, or/and directed. Another important distinction is whether the communities are allowed to overlap or not. In the overlapping communities case, each vertex can belong to several subsets. 

A difficulty with community detection is that the notion of 
community is not well defined. Different algorithms may employ 
different formal notions of a community, and can sometimes produce 
different results. Nevertheless, there exist several widely adopted benchmarks -- 
synthetic models and real-life graphs -- where 
the ground truth communities are known, and algorithms are evaluated 
based on the
similarity of the produced output to the ground truth, and based on 
the amount of 
required computations.  On the theoretical side, most of the effort 
is concentrated on developing 
algorithms with guaranteed recovery of clusters for graphs generated 
from variants of the Stochastic Block 
Model (referred to as SBM in what follows, \cite{Fortunato201075}). 

In this paper we present a new algorithm, DER (Diffusion Entropy Reducer, for reasons to be clarified later), for non-overlapping 
community detection. The algorithm is an adaptation of the k-means algorithm to a space of measures 
which are generated by short random walks from the nodes of the graph. The adaptation is done by introducing a 
certain natural cost on the space of the measures. As detailed below, we evaluate the DER on several 
benchmarks and find its performance to be as good or better than the best alternative method. In addition, 
we establish some theoretical guarantees on its performance. While the main purpose of the theoretical 
analysis in this paper is to provide some insight into why DER works, our result is also one of the very few 
results in the literature that show reconstruction in linear time. 

On the empirical side, we first evaluate our algorithm on a set of random graph benchmarks known as the LFR 
models, \cite{LFBench}. In \cite{LFComp}, 12 other algorithms were evaluated on these benchmarks, and three 
algorithms, described in \cite{Infomap}, \cite{RN_method} and \cite{blondel2008}, were identified, that exhibited 
significantly better performance than the others, and similar performance among themselves. We evaluate our 
algorithm on random graphs with the same parameters as those used in \cite{LFComp} and find its performance to 
be as good as these three best methods. Several well known methods, including spectral clustering 
\cite{NJW}, exhaustive modularity optimization (see \cite{LFComp} for details), and clique percolation 
\cite{cfinder}, have worse performance on the above benchmarks. 

Next, while our algorithm is designed for non-overlapping communities, we introduce a simple modification that 
enables it to detect overlapping communities in some cases. Using this modification, we compare the 
performance of our algorithm to the performance of 4 overlapping community algorithms on a set of benchmarks 
that were considered in \cite{GopBlei}. We find that in all cases DER performs better than all 4 algorithms. 
None of the algorithms evaluated in \cite{LFComp} and \cite{LFBench} has theoretical guarantees.

On the theoretical side, we show that DER reconstructs with high probability the partition of 
the $p,q$-stochastic block model such that, roughly, 
$p \geq N^{-\half}$, where $N$ is the number of vertices, and $p-q \geq c \sqrt{p N^{-\half +  \epsilon} \log N}$ (this holds in particular when $\frac{p}{q} \geq c'>1$) for some 
constant $c>0$. We show that for this reconstruction only one iteration of the $k$-means is sufficient. In 
fact, three passages over the set of edges suffice. While the cost function we introduce for DER will appear 
at first to have purely probabilistic motivation, for the purposes of the proof we provide an alternative 
interpretation of this cost in terms of the graph, and the arguments show which properties of the graph are 
useful for the convergence of the algorithm. 

Finally, although this is not the emphasis of the present paper,
it is worth noting here that, as will be evident later, our algorithm 
can be trivially parallelalized. This seems to be a particularly nice 
feature since most other algorithms, including spectral clustering, are not easy to parallelalize and do not 
seem to have parallel implementations at present.

The rest of the paper is organized as follows: Section \ref{lit_sec} overviews 
related work and discusses relations to our results. In Section \ref{algo_sec} we 
provide the motivation for the definition of the algorithm, derive the cost 
function and establish some basic properties. Section \ref{empirical_sec} we 
present the results on the empirical evaluation of the algorithm and 
Section \ref{analytic_sec} describes the theoretical guarantees and the general proof scheme. 
Some proofs and additional material are provided in the supplementary material.

\section{Literature review}
\label{lit_sec}

Community detection in graphs has been an active research topic for the last two decades 
and generated a huge literature. We refer to ~\cite{Fortunato201075} for an 
extensive survey. Throughout the paper, let $G=(V,E)$ be a graph, and let $P=P_1,\ldots,P_k$ be a 
partition of $V$. Loosely speaking, a partition $P$ is a good community structure 
on $G$ if for each $P_i \in P$, more edges stay within $P_i$ than leave $P_i$. This is usually quantified 
via some cost function that assigns larger scalars to partitions $P$ that are in some 
sense better separated. Perhaps the most well known cost function is the modularity, which was introduced in 
\cite{GNFoot} and served as a basis of a large number of community detection algorithms 
(\cite{Fortunato201075}). The popular spectral clustering methods,  \cite{NJW}; \cite{vL}, can also 
be viewed as a (relaxed)  optimization of a certain cost (see \cite{vL}). 

Yet another group of algorithms is based on fitting a generative model of a graph with
communities to a given graph. References \cite{newman2007mixture}; \cite{GopBlei} are two 
among the many examples. Perhaps the simplest generative model for non-overlapping 
communities is the stochastic block model, see \cite{PPM},\cite{Fortunato201075} which we now 
define: Let $P=P_1,\ldots,P_k$ be a partition of $V$ into $k$ subsets. $p,q$-SBM is 
a distribution over the graphs on vertex set $V$ , such that all edges are 
independent and for $i,j \in V$, the edge $(i,j)$ exists with probability $p$ if 
$i,j$ belong to the same $P_s$, and it exists with probability $q$ otherwise. If 
$q << p$, the components $P_i$ will be well separated in this model. 
We denote the number 
of nodes by $N=|V|$ throughout the paper.

Graphs generated from SBMs can serve as a benchmark for community detection algorithms. However, such graphs 
lack certain desirable properties, such as power-law degree and community size distributions. Some of these 
issues were fixed in the benchmark models in \cite{LFBench}; \cite{LFRBench}, and these models are referred to 
as LFR models in the literature. More details on these models are given in Section \ref{empirical_sec}. 

 We now turn to the discussion of the theoretical guarantees. Typically results in this 
direction provide algorithms that can reconstruct,with high probability, the ground 
partition of a graph drawn from a variant of a $p,q$-SBM model, with some, possibly 
large, number of components $k$. Recent results include the works \cite{anandkumar_tensor} and \cite{Chen}. In this paper, however, we shall only analytically analyse the $k=2$ case, and such that, in 
addition, $|P_1|=|P_2|$. 

 For this case, the best known reconstruction result was obtained already in \cite{Bopanna} and was only 
improved in terms of runtime since then. Namely, Bopanna's result states that if $p \geq 
c_1\frac{\log N}{N}$ and $p-q \geq c_2 \frac{\log N}{N}$, then with high probability the partition is 
reconstructible.   Similar bound can be obtained, for instance, from the approaches in 
\cite{anandkumar_tensor}; \cite{Chen}, to name a few.  The methods in this group are generally based on the 
analysis of the spectrum of the adjacency matrix. The run time of these algorithms is non-linear 
in the size of the graph and it is not known how these algorithms behave on graphs not generated by the 
probabilistic models that they assume. 

It is generally known that when the graphs are dense ($p$ of order of constant), simple linear time 
reconstruction algorithms exist (see \cite{SBB}). The first, and to the best of our 
knowledge, the only previous linear time algorithm for non dense graphs was proposed in \cite{SBB}. This algorithm works for $p \geq c_3(\epsilon) N^{-\half + \epsilon}$, for 
any fixed $\epsilon > 0$. The approach of \cite{SBB} was further extended in \cite{ShamirTsur}, to handle more 
general cluster sizes. These approaches approaches differ significantly from the spectrum based 
methods, and provide equally important theoretical insight. However, their empirical behaviour was never 
studied, and it is likely that even for graphs generated from the SBM, extremely high values of $N$ would be 
required for the algorithms to work, due to large constants in the concentration 
inequalities (see the concluding remarks in \cite{ShamirTsur}). 

\section{Algorithm}
\label{algo_sec}

Let $G$ be a finite undirected graph with a vertex set $V = \{1,\ldots, n\}$. Denote by $A = \{a_{ij}\}$ the symmetric adjacency matrix of $G$, where $a_{ij} \geq 0$ are edge weights, and for a vertex $i \in V$, set $d_i = \sum_{j} a_{ij}$ to be the degree of $i$. Let $D$ be an $n \times n$ diagonal matrix such that $D_{ii} = d_i$, and set $T = D^{-1} A$ to be the transition matrix of the random walk on $G$. Set also $p_{ij} = T_{ij}$. Finally, denote by $\pi$, $\pi(i) = \frac{d_i}{\sum_j d_j}$ the stationary measure of the random walk. 

  A number of community detection algorithms are based on the intuition that distinct communities should be relatively closed under the random walk (see \cite{Fortunato201075}), and  employ different notions of closedness. Our approach also takes this point of view. 
  
  For a fixed $L \in \NN$, consider the following sampling process on the graph: Choose vertex $v_0$ randomly from $\pi$, and perform $L$ steps of a random walk on $G$, starting from $v_0$. This results in a length $L+1$ sequence of vertices, $x^{1}$. Repeat the process $N$ times independently, to obtain also $x^1,\ldots,x^N$. 
  
  Suppose now that we would like to model the sequences $x^s$ as a multinomial mixture model with a single component. Since each coordinate $x^s_t$ is distributed according to $\pi$, the single component of the mixture  should be $\pi$ itself, when $N$ grows.  Now suppose that we would like to model the same sequences with a mixture of two components. Because the sequences are sampled from a random walk rather then independently from each other, the components need no longer be $\pi$ itself, as in any mixture where some elements appear more often together then others.  
  The mixture as above can be found using the EM algorithm, and this in principle summarizes our approach. The only additional step, as discussed above, is to replace the sampled random walks with their true distributions, which simplifies the analysis and also leads to somewhat improved empirical performance. 
 
  We now present the DER algorithm for detecting the non-overlapping
  communities. Its input is the number of components to detect, $k$, the length of 
  the walks $L$, an initialization partition 
  $P =\left\{ P_1,\ldots,P_k \right\}$ of $V$ into disjoint subsets. $P$ would be 
  usually taken to be a random partition of $V$ into equally sized subsets.

For $t = 0,1,\ldots $ and a vertex $i \in V$, denote by $w_i^t$ the $i$-th row of 
the matrix $T^t$. Then $w^t_i$ is the distribution of the random walk on $G$, started at $i$, after $t$ 
steps.
Set $w_i = \frac{1}{L} (w_i^1 + \ldots + w_i^L)$, which is the distribution 
corresponding to the average of the empirical measures of sequences $x$ that start 
at $i$. 

For two probability measures $\nu,\mu$ on $V$, set 
\[
	D(\nu,\mu) = \sum_{i\in V} \nu(i) \log \mu(i). 
\]
Although $D$ is not a metric, will act as a distance function in our algorithm. Note that if $\nu$ was an 
empirical measure, then, up to a constant, $D$ would be just the log-likelihood of 
observing $\nu$ from independent samples of $\mu$. 

For a subset $S \subset V$, set $\pi_S$ to be the restriction of the measure $\pi$ 
to $S$, and also set $d_S = \sum_{i \in S} d_i $ to be the full degree of $S$.  Let 
\begin{equation}
\label{eq:mu_s_def}
\mu_S = \oneover{d_S} \sum_{i \in S} d_i w_i
\end{equation}
denote the distribution of the random walk started from $\pi_S$. 

The complete DER algorithm is described in Algorithm \ref{alg:noc}.
\begin{algorithm}[tb]
   \caption{DER}
   \label{alg:noc}
\begin{algorithmic}[1]
   \STATE {\bfseries Input:} Graph $G$, walk length $L$, \\ 
   		\spaceo \spaceo number of components $k$.
   \STATE Compute the measures $w_i$.
   \STATE Initialize $P_1,\ldots,P_k$ to be a random partition such that \\
          \spaceo \spaceo  $|P_i| = |V| / k $ for all $i$. 
   
   \REPEAT
   \STATE (1) \spaceo For all $s\leq k$, construct $\mu_s = \mu_{P_s}$.
   \STATE (2) \spaceo For all $s\leq k$, set \\
   \[
   		P_s =  \left\{ i \in V \setsep s = \argmax_l D(w_i, \mu_l ) \right\}.
   \]
   \UNTIL{the sets $P_s$ do not change}
\end{algorithmic}
\end{algorithm}

The algorithm is essentially a k-means algorithm in a non-Euclidean space, where the points are the measures $w_i$, each occurring with multiplicity $d_i$. Step (1) is the ``means" step, and (2) is the maximization step. 

Let 
\begin{equation}
	C = \sum_{l = 1}^{L} \sum_{i \in P_l} d_i \cdot D(w_i, \mu_l)
\end{equation}
be the associated cost. As with the usual k-means, we have the following 
\begin{lem} 
\label{lem:monotonicity}
Either $P$ is unchanged by steps (1) and (2) or both steps (1) and (2) strictly increase the value of $C$. 
\end{lem}
The proof is by direct computation and is deferred to the supplementary material. Since the number of configurations $P$ is finite, it follows that DER always terminates and provides a ``local maximum" of the cost $C$.

  The cost $C$ can be rewritten in a somewhat more informative form. To do so, we introduce some notation first. 
Let $X$ be a random variable on $V$, distributed according to measure $\pi$. Let $Y$ a step of a random walk started at $X$, so that the distribution of $Y$ given $X = i$ is $w_i$. Finally, for a partition $P$, let $Z$ be the indicator variable of a partition, $Z = s$ iff $X \in P_s$.  With this notation, one can write 
\begin{equation}
\label{algo_info_eq}
C = -d_V \cdot H(Y |Z ) = d_V \left( -H(Y) + H(Z) - H(Z|Y) \right),
\end{equation}
where $H$ are the full and conditional Shannon entropies.  Therefore, DER algorithm can be interpreted as 
seeking 
a partition that maximizes the information between current known state ($Z$), and the next step from it ($Y$). This interpretation gives rise to the name of the algorithm, DER, since every iteration reduces the entropy $H(Y|Z)$ of the random walk, or diffusion, with respect to the partition. 
The second equality in (\ref{algo_info_eq}) has another interesting interpretation. Suppose, for simplicity, 
that 
$k=2$, with partition $P_1,P_2$. In general, a clustering algorithm aims to minimize the cut, the number of 
edges 
between $P_1$ and $P_2$. However, minimizing the number of edges directly will lead to situations where $P_1$ 
is 
a single node, connected with one edge to the rest of the graph in $P_2$. To avoid such situation, a relative, 
normalized version of a cut needs to be introduced, which takes into account the sizes of $P_1,P_2$. Every 
clustering algorithms has a way to resolve this issue, implicitly or explicitly. For DER, this is shown in 
second equality of (\ref{algo_info_eq}). $H(Z)$ is maximized when the components are of equal sizes (with 
respect to 
$\pi$), while $H(Z|Y)$ is minimized when the measures $\mu_{P_s}$ are as disjointly supported as possible.

As any $k$-means algorithm, DER's results depend somewhat on its random 
initialization. All $k$-means-like schemes are usually restarted several times and the solution with the best 
cost is chosen. In all cases which we evaluated we observed empirically that the dependence of DER on the 
initial parameters is rather weak. After two or three restarts it usually found a partition nearly as good as 
after 100 restarts. For clustering problems, however, there is another simple way to aggregate the results of 
multiple runs into a single partition, which slightly improves the quality of the final results. We use this 
technique in all our experiments and we provide the details in the Supplementary Material, Section \AppSecRestarts.

We conclude by mentioning two algorithms that use some of the concepts that we use.  The Walktrap, \cite{Pons}, 
similarly to DER constructs the random walks (the measures $w_i$, possibly for $L>1$) as part of its 
computation. However, Walktrap uses $w_i$'s in a completely different way. Both the 
optimization procedure and the cost function are different from ours.  The Infomap , \cite{Infomap}, \cite{InfomapOvr}, has 
a cost that is related to the notion of information. It aims to minimize to the information required to 
transmit a random walk on $G$ through a channel, the source coding is constructed using the clusters, and best 
clusters are those that yield the best compression. This does not seem to be directly connected to the maximum 
likelyhood motivated approach that we use. As with Walktrap, the optimization procedure of Infomap also 
completely differs from ours. 

\section{Evaluation}
\label{empirical_sec}
   In this section results of the evaluation of DER algorithm are presented. 
In Section \ref{small_sec} we illustrate DER on two classical graphs. Sections \ref{lfrbno_sec} and \ref{sec:overlap_lfr_bench} contain the evaluation on the LFR benchmarks.

\subsection{Basic examples}
\label{small_sec}

\begin{figure}
\centering
\begin{subfigure}{.5\textwidth}
  \centering
  \includegraphics[width=\textwidth]{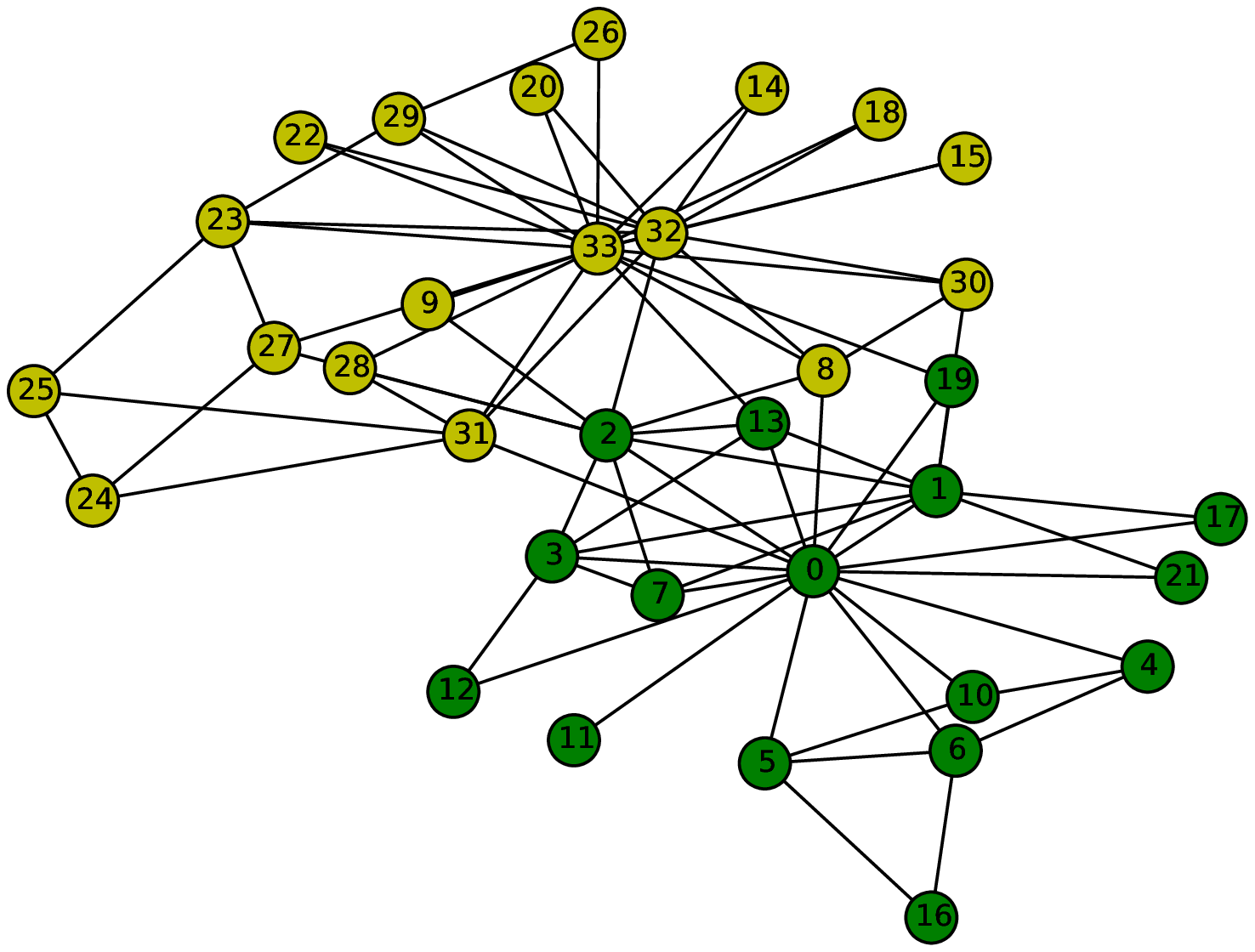}
  \caption{Karate Club}
  \label{fig:karate}
\end{subfigure}%
\begin{subfigure}{.5\textwidth}
  \centering
  \includegraphics[width=\textwidth]{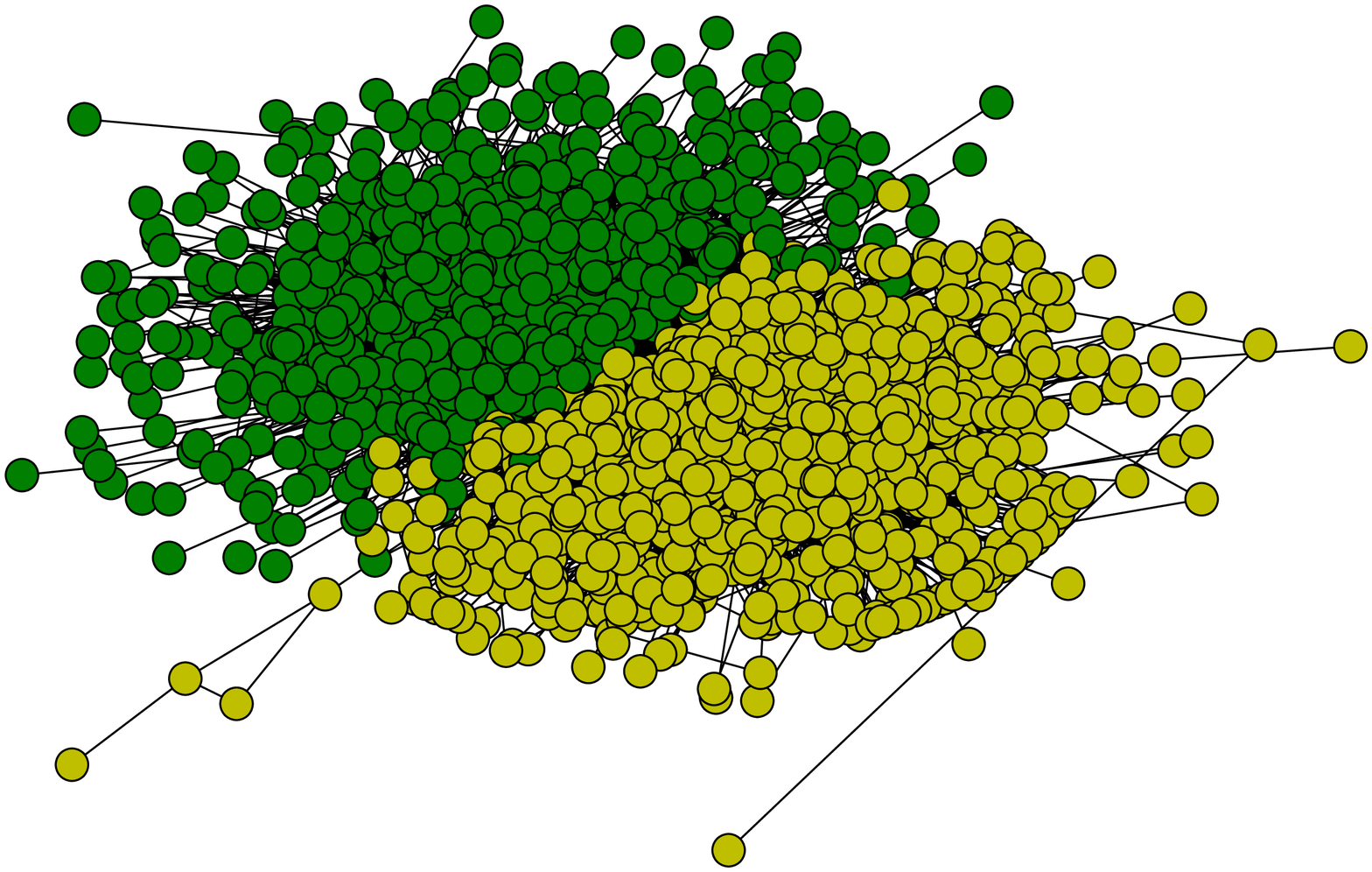}
  \caption{Political Blogs}
  \label{fig:political_blogs}
\end{subfigure}
\end{figure}

When a new clustering algorithm is introduced, it is useful to get a general feel of it with some simple 
examples. Figure~\ref{fig:karate} shows the classical Zachary's Karate Club, \cite{Zachary}. This graph has a 
ground partition into two subsets. The partition shown in Figure \ref{fig:karate} is a partition obtained from 
a typical run of DER algorithm, with $k=2$, and wide range of  $L$'s. ($L \in [1,10]$ were tested). As is the 
case with many other clustering algorithms, the shown partition differs from the ground partition in one 
element, node $8$ (see \cite{Fortunato201075}).

Figure~\ref{fig:political_blogs} shows the political blogs graph, \cite{PolBlogs}.  The nodes are political 
blogs, and the graph has an (undirected) edge if one of the blogs had a link to the other. There are 1222 
nodes in the graph. The ground truth partition of this graph has two components - the right wing and left wing 
blogs. The labeling of the ground truth was partially automatic and partially manual, and both processes could 
introduce some errors. The run of DER reconstructs the ground truth partition with only 57 nodes 
missclassifed. The NMI (see the next section, Eq. (\ref{eq:nmi_def})) to the ground truth partition is $.74$. 

The political blogs graphs is particularly interesting since it is an example of a graph for which fitting an SBM model to reconstruct the clusters produces results very different from the ground truth. It can also be easily checked that spectral clustering, in form given in \cite{NJW}, is not close to ground truth when $k=2$. It is 
close to ground truth when $k=3$, however. To overcome the problem with SBM fitting on this graph, a degree sensitive version of SBM was introduced in ~\cite{NewmanSBM}. That algorithm produces partition with NMI $.75$.

\subsection{LFR benchmarks}
\label{lfrbno_sec}

The LFR benchmark model, \cite{LFRBench}, is a widely used extension of the stochastic block model, where node degrees 
and community sizes have power law distribution, as often observed in real graphs. An important parameter of 
this model is the mixing parameter $\mu \in [0,1]$ that controls the fraction of the edges of a node that 
go outside the node's community (or outside all of node's communities, in the overlapping case). For small 
$\mu$, there will be a small number of edges going outside the communities, leading to disjoint, easily 
separable graphs, and the boundaries between communities will become less pronounced as $\mu$ grows.  

Given a set of communities $P$ on a graph, and the ground truth set of communities $Q$, there are several 
ways to measure how close $P$ is to $Q$. One standard measure is the normalized mutual information (NMI), given by:
\begin{equation}
\label{eq:nmi_def}
	NMI(P,Q) = 2\frac{I(P,Q)}{H(P) + H(Q)},
\end{equation}
where $H$ is the Shannon entropy of a partition and $I$ is the mutual information (see \cite{Fortunato201075} 
for details). NMI is equal $1$ if and only if the partitions $P$ and $Q$ coincide, and it takes values between $0$ and $1$ otherwise. 

When computed with NMI, the sets inside $P,Q$ can not overlap. To deal with overlapping communities, an extension of NMI was 
proposed in \cite{LFK_nmi}. We refer to the original paper for the definition, as the definition is somewhat 
lengthy. This extension, which we denote here as ENMI, was subsequently used in the literature as a measure 
of closeness of two sets of communities, event in the cases of disjoint communities. Note that most papers 
use the notation NMI while the metric that they really use is ENMI. 

Figure \ref{fig:lfr_nmi} shows the results of evaluation of DER for four cases: 
the size of a graph was either $N=1000$ or $N=5000$ nodes, and the size of the communities was restricted 
to be either between $10$ to $50$ (denoted $S$ in the figures) or between $20$ to $100$ (denoted $B$). For 
each combination of these parameters, $\mu$ varied between $0.1$ and $0.8$. For each combination of graph 
size, community size restrictions as above and $\mu$ value, we generated 20 graphs from that model and run 
DER. To provide some basic intuition about these graphs, we note that the number of communities in the 1000S 
graphs is strongly concentrated around 40, and in 1000B, 5000S, and 5000B graphs it is around 25, 200 and 100 
respectively. Each point in Figure \ref{fig:lfr_nmi} is a the average ENMI on the 20 corresponding graphs, with standard deviation as the error bar. These experiments correspond precisely to the ones performed 
in \cite{LFComp} (see Supplementary Material, Section \AppSecLFR for more details). In all runs on DER we 
have set L = 5 and set $k$ to be the true number of communities for each graph, as was done in \cite{LFComp} for the methods that required it. Therefore our Figure \ref{fig:lfr_nmi} can be compared directly with Figure 2 in \cite{LFComp}.

\begin{figure}
\centering
\begin{subfigure}{.5\textwidth}
    \centering
	\centerline{\includegraphics[width=6cm, height=4cm]{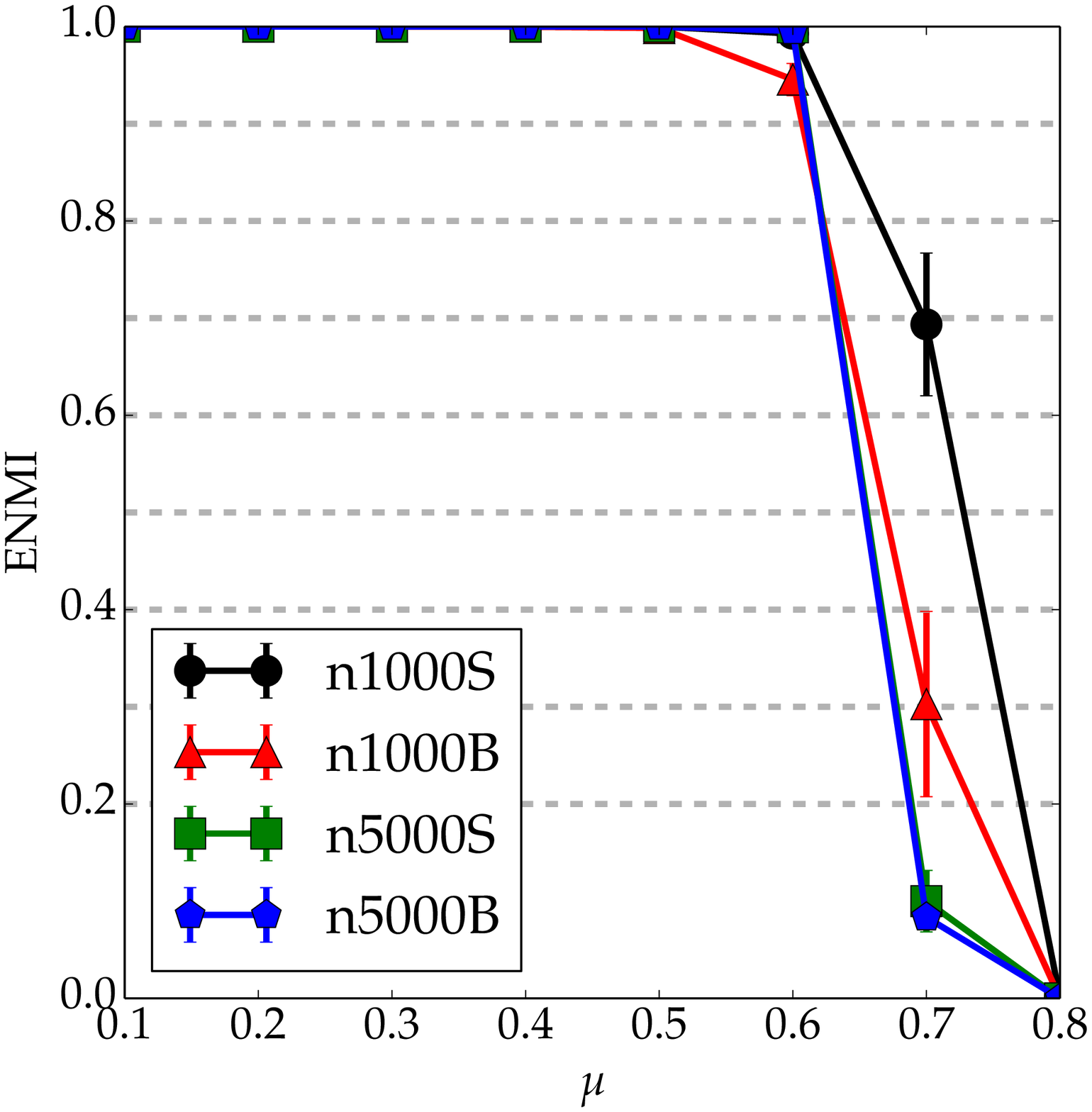}}
	\caption{DER, LFR benchmarks}
	\label{fig:lfr_nmi}
\end{subfigure}%
\begin{subfigure}{.5\textwidth}
	\centering
	\includegraphics[width=6cm, height=4cm]{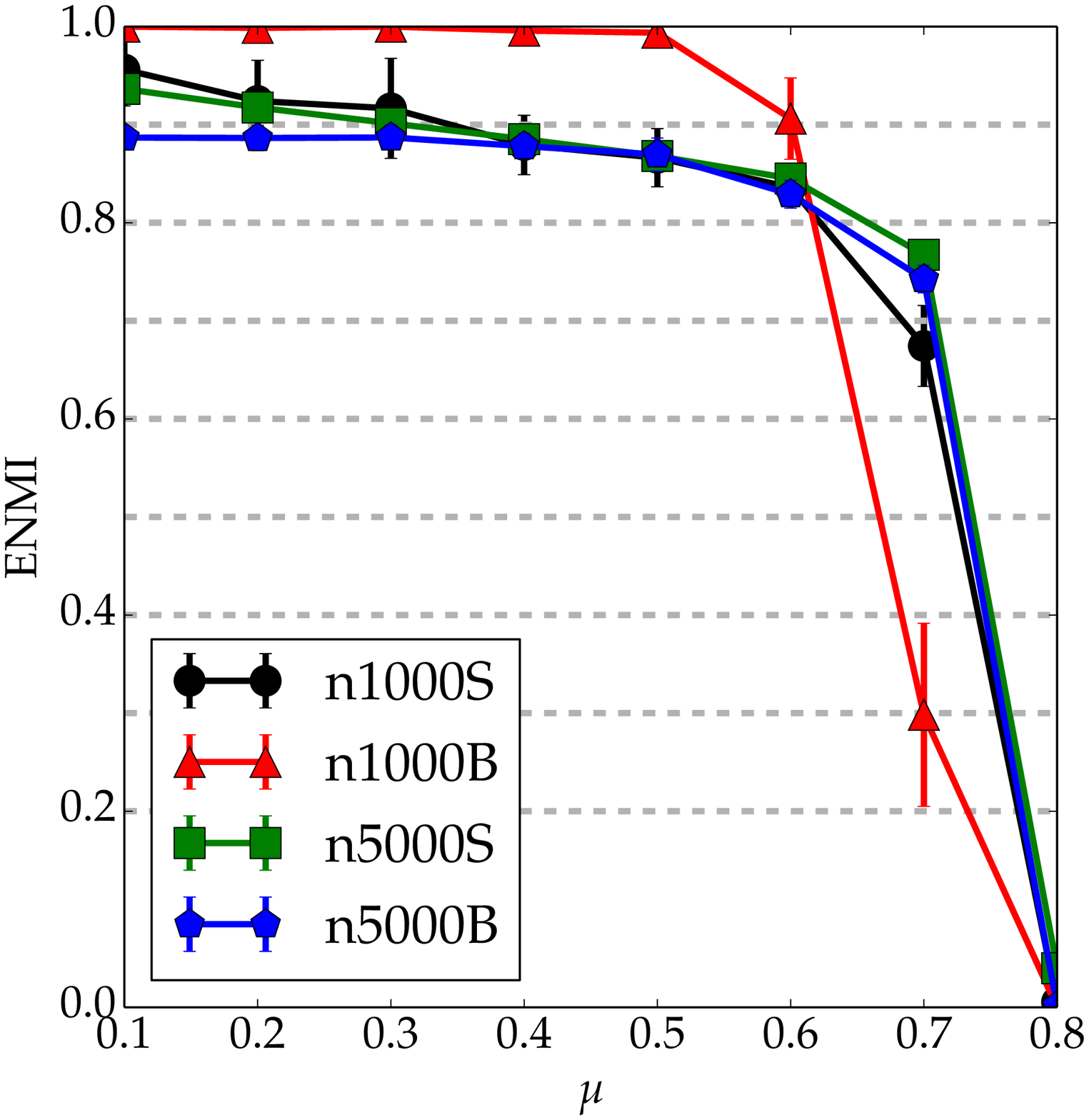}
	\caption{Spectral Alg., LFR benchmarks}
	\label{fig:spec_lfr_enmi}
\end{subfigure}
\end{figure}

From this comparison we see that DER and the two of the best algorithms identified in \cite{LFComp}, Infomap 
\cite{Infomap} and RN \cite{RN_method}, reconstruct the partition perfectly for $\mu \leq 0.5$, for $\mu = 
0.6$ DER's reconstruction scores are between Infomap's and RN's, with values for all of the algorithms above 
$0.95$, and for $\mu = 0.7$ DER has the best performance in two of the four cases. For $\mu = 0.8$ all 
algorithms have score $0$. 

We have also performed the same experiments with the standard version of spectral clustering, \cite{NJW}, 
because this version was not evaluated in \cite{LFComp}. The results are shown in Fig. 
\ref{fig:spec_lfr_enmi}. Although the performance is generally good, the scores are mostly lower than those 
of DER, Infomap and RN. 

\subsection{Overlapping LFR benchmarks}

\label{sec:overlap_lfr_bench}

We now describe how DER can be applied to overlapping community detection.  Observe that DER internally operates on measures $\mu_{P_s}$ rather then subsets of the vertex set. Recall 
that $\mu_{P_s}(i)$ is the probability that a random walk started from $P_s$ will hit node $i$. We can therefore consider each $i$ to be a member of those communities from which the probability to hit it is ``high enough".  To define this formally, we first note that for any partition $P$, the following decomposition holds:
\begin{equation}
\label{algo_meas_dec}
	\pi = \sum_{s=1}^k \pi(P_s) \mu_{P_s}.
\end{equation}
This follows from the invariance of $\pi$ under the random walk. Now, given the out put of DER - the sets $P_s$ and measures $\mu_{P_s}$ set
\begin{equation}
\label{algo_community_membership_meas}
	m_i(s) = \frac{\mu_{P_s}(i) \pi(P_s)}{\sum_{t =1}^k \mu_{P_t}(i) \pi(P_t)} = \frac{\mu_{P_s}(i) \pi(P_s)}{\pi(i)}, 
\end{equation}
where we used (\ref{algo_meas_dec}) in the second equality. Then $m_i(s)$ is the probability that the walks started at $P_s$, given that it finished in $i$. For each $i \in V$, set $s_i = \argmax_l m_i(l)$ to be the most likely community given $i$. Then define the overlapping communities $C_1,\ldots,C_k$ via
\begin{equation}
	\label{eq:ovlp_comm_choice}
   		C_t = \left\{ i \in V \setsep m_i(t) \geq \half \cdot m_i(s_i) \right\}.
\end{equation}

The paper \cite{GopBlei} introduces a new algorithm for overlapping communities detection and contains also 
an evaluation of that algorithm as well as of several other algorithms on a set of overlapping LFR 
benchmarks. The overlapping communities LFR model was defined in \cite{LFBench}. In Table \ref{tab:eval_settings} we present the ENMI results of DER runs on the $N=10000$ graphs 
with same parameters as in \cite{GopBlei}, and also show the values obtained on these benchmarks in 
\cite{GopBlei} (Figure S4 in \cite{GopBlei}), for four other algorithms. The DER algorithm was run with $L=2$, and $k$ 
was set to the true number of communities. Each number is an average over ENMIs on 10 instances of graphs with 
a given set of parameters (as in \cite{GopBlei}).  The standard deviation around this average for DER was less then 
$0.02$ in all cases. Variances for other algorithms are provided in \cite{GopBlei}. 

\begin{table}
\centering
\caption{Evaluation for Overlapping LFR. All values except DER are from \cite{GopBlei}}
\label{tab:eval_settings}
\begin{tabular}{|l|l|l|l|} \hline
Alg. & $\mu = 0$ & $\mu = 0.2$ & $\mu=0.4$ \\ \hline
DER & \textbf{0.94} & \textbf{0.9} & \textbf{0.83} \\ \hline
SVI (\cite{GopBlei}) & 0.89 & 0.73 & 0.6 \\ \hline
POI (\cite{NewmanPoi})& 0.86 & 0.68 & 0.55 \\ \hline
INF (\cite{InfomapOvr})& 0.42 & 0.38 & 0.4 \\ \hline
COP (\cite{Greg})& 0.65 & 0.43 & 0.0 \\ \hline
\end{tabular}
\end{table}

For $\mu \geq 0.6$ all algorithms yield ENMI of less then $0.3$. As we see in Table \ref{tab:eval_settings}, 
DER performs better than all other algorithms in all the cases. We believe this indicates that DER together 
with equation (\ref{eq:ovlp_comm_choice}) is a good choice for overlapping community detection in situations 
where community overlap between each two communities is sparse, as is the case in the LFR models considered 
above. Further discussion is provided in the Supplementary Material, Section 
\AppSecLFROverlap.

We conclude this section by noting that while in the non-overlapping case the models generated with $\mu=0$
result in trivial community detection problems, because in these cases communities are simply the connected components of the graph, this is no longer true in the overlapping case.  As a point of reference, the well known Clique 
Percolation method was also evaluated in \cite{GopBlei}, in the $\mu=0$ case. The average 
ENMI for this algorithm was $0.2$ (Table S3 in \cite{GopBlei}).

\section{Analytic bounds}
\label{analytic_sec}

\newcommand{\rgraph}[1]{\mathcal{#1}}
\newcommand{\BRK}[1]{\left(#1\right)}
\newcommand{\Bin}[2]{Bin\left(#1,#2\right)}

In this section we restrict our attention to the case $L=1$ of the DER algorithm. 
Recall that the $p,q$-SBM model was defined in Section \ref{lit_sec}. We shall consider the model with $k=2$ 
and such that  $|P_1| =|P_2|$.  We assume that the initial partition for the DER, denoted $C_1,C_2$  
in what follows, is chosen as in step 3 of DER (Algorithm \ref{alg:noc}) - a random partition of $V$ into two 
equal sized subsets. 

In this setting we have the following:
\begin{thm}
\label{thm:pq_bound}
For every $\epsilon >0$ there exists $C > 0$ and $c>0$ such that if
\begin{equation}
\label{eq:p_size}
	p \geq  C \cdot N^{-\half + \epsilon}
\end{equation}
and 
\begin{equation}
\label{eq:q_size}
	p-q \geq c \sqrt{p N^{-\half +  \epsilon} \log N}
\end{equation}
then DER recovers the partition $P_1,P_2$ after one iteration, with probability $\phi(N)$ such that $\phi(N) \rightarrow 1$ when $N \rightarrow \infty$. 
\end{thm}

Note that the probability in the conclusion of the theorem refers to a joint probability of a draw from the 
SBM and of an independent draw from the random initialization. 

The proof of the theorem has essentially three steps. First, we observe that the random initialization 
$C_1,C_2$ is necessarily somewhat biased,  in the sense that $C_1$ and $C_2$ never divide $P_1$ exactly into 
two halves. Specifically, $\left| |C_1 \cap P_1| - |C_2 \cap P_1| \right| \geq N^{-\half - \epsilon}$ with 
high probability. Assume that $C_1$ has the bigger half, $|C_1 \cap P_1| > |C_2 \cap P_1|$. In the second 
step, by an appropriate linearization argument we show that for a node $i \in P_1$, deciding whether $D(w_i,
\mu_{C_1}) > D(w_i,\mu_{C_2})$ or vice versa amounts to counting paths of length two between $i$ and $|C_1 
\cap P_1|$. In the third step we estimate the number of these length two paths in the model. The fact that $|
C_1 \cap P_1| > |C_2 \cap P_1| + N^{-\half - \epsilon}$ will imply more paths to $C_1 \cap P_1$ from $i \in 
P_1$ and we will conclude that $D(w_i,\mu_{C_1}) > D(w_i,\mu_{C_2})$ for all $i \in P_1$ and $D(w_i,\mu_{C_2}) 
> D(w_i,\mu_{C_1})$ for all $i \in P_2$. The full proof is provided in the supplementary 
material.

We note that the use of paths of length two is essential for the argument to work. Similar argument with 
paths of length one (edges) will not work (unless $p$ is of the order of a constant). However, we also note 
that paths of length two are never explicitly computed, as this would require squaring the adjacency matrix. 
Instead, this is achieved by considering paths of length one from the target set $C_1$ (via $\mu_{C_1}$) and 
paths of length one from the nodes (via $w_i$).

\bibliographystyle{unsrt}
\bibliography{communities}

\newpage
\appendix

\section{Restarts and repeats}

\label{sec:repeats_and_restarts}

As any $k$-means algorithm, DER's results depend somewhat on its random initializations, and can be improved 
by multiple runs on the same instance with different initializations.  We refer to this as restarts of the 
algorithm. We have observed empirically the following behaviour of DER: Suppose a graph $G$ has a 
ground truth partition $P_1,\ldots,P_k$. Then the output of a typical restart of DER will be a partition 
$C_1,\ldots,C_k$ with the property that for each $C_i$,$i \leq k$, either there is $j \leq k$ such that 
$C_i = P_j$, or there are $j_1,j_2$ such that $C_i = P_{j_1} \cup P_{j_2}$ or there are $j$ and $l$ such that 
$C_i \cup C_l = P_j$. In other words, DER tends to either find the precise cluster, or to glue together two 
original clusters, or split an original cluster into two parts.  Usually most of the clusters will be found 
precisely, and there will be some small number of (usually small) clusters that are glued or splitted. Which 
clusters will be glued or splitted would depend on the random initialization. An simple way to deal with this 
is to use the following ``repeats" strategy: Choose a number of repeats, $R$ (say, $R=5$) and run DER $R$ 
times. Construct the node co-occurence matrix: 
\begin{equation}
\label{eq:co_oc_matrix}
	\hat{R_{ij}} = \mbox{ number of runs such that $i$ and $j$ appear in the same cluster.}
\end{equation}
for all $i,j \in V$. 

The matrix $\hat{R}$ can now be regarded as an adjacency matrix of a weighted graph and 
can be clustered itself. However, $\hat{R}$ will often have very clear clusters, which can be found using the 
following trivial threshold algorithm:  Define $T = \ceil{R/2}$. Initialize a set $U = V$.  Choose an 
arbitrary $i \in U$ and define a cluster $C$ by 
\[
	C = \SetBrack{ j \in U \setsep \hat{R}_{ij} \geq T }.
\]
Then output cluster $C$, set $U = U \setminus C$, choose a new $i \in U$ and repeat until $U$ is empty.

While on the benchmarks a single run of DER with a single restart usually has quite high precision, repeats 
are a more effective way to deal with glueing and splitting than the restarts. It is of course also possible 
to use more sophisticated but slower algorithms instead of the threshold one to cluster the co-occurence 
matrix $R$.

\section{Proofs}

\subsection{Lemma 3.1}
\begin{proof}[Proof Of Lemma 3.1:]
The claim is obvious for step (2) of the algorithm. For step (1) the claim is implied by the following standard fact:  Let $\nu_1,\nu_2,\ldots,\nu_z$ be any finite collection of measures. Set $\tilde{\nu} = \frac{1}{z} \sum_{i} \nu_i$. 
Then for any measure $\kappa$, 
\begin{equation}
\label{eq:maxlikeopt} 
    \sum_{i = 1}^z D(\nu_i,\kappa) \leq \sum_{i = 1}^z D(\nu_i,\tilde{\nu}).
\end{equation}
Indeed, by rearranging terms in (\ref{eq:maxlikeopt}), we get
\begin{eqnarray*}
    \sum_{j \in V} \left(\sum_{i=1}^z \nu_i(j) \right) \left( 
        \log \tilde{\nu}(j)  - \log \kappa(j)
        \right) = \\
    z \cdot \sum_{j \in V} \tilde{\nu}(j) \left( 
        \log \frac{\tilde{\nu}(j)}{\kappa(j)}
        \right)  \geq 0 
\end{eqnarray*} 
which is the non-negativity of the Kullback-Leibler divergence ~\cite{Cover2006}, with equality iff $\kappa = \tilde{\nu}$. 
\end{proof}

\subsection{Main result}
\newcommand{\PapMainThm}{5.1}
\newcommand{\PapMainThmSection}{5}

We now prove Theorem \PapMainThm, which we restate here for convenience. 

\begin{thm}
\label{thm:pq_bound}
For every $\epsilon >0$ there exists $C > 0$ and $c>0$ such that if
\begin{equation}
\label{eq:p_size}
	p \geq  C \cdot N^{-\half + \epsilon}
\end{equation}
and 
\begin{equation}
\label{eq:q_size}
	p-q \geq c \sqrt{p N^{-\half +  \epsilon} \log N}
\end{equation}
then DER recovers the partition $P_1,P_2$ after one iteration, with probability $\phi(N)$ such that $\phi(N) \rightarrow 1$ when $N \rightarrow \infty$. 
\end{thm}

Recall that a general plan of the proof was discussed in Section \PapMainThmSection. We proceed 
to implement that plan.  We start with stating some preliminaries. 
First, we state a version of Chernoff's bound for binomial variables.  
\begin{thm}[Theorem 2.1 in \cite{janson}]
\label{thm:chernoff_main} 
Let $X \sim Bin(n,p)$ be a binomial variable and set $\lambda = np$. Then for all $t\geq 0$,

\begin{equation}
\label{eq:bin_upper_tail}
\Prob{X \geq \Exp X  + t } \leq \exponent{ - \frac{t^2}{2(\lambda + t/3)} }
\end{equation}

\begin{equation}
\label{eq:bin_lower_tail}
\Prob{X \leq \Exp X  - t } \leq \exponent{ - \frac{t^2}{2\lambda} }
\end{equation}

\end{thm}

In general given a binomial $X \sim Bin(n,p)$ we will often refer to $\lambda = np$ as $X$'s lambda.

The following Corollary will be useful. 
\begin{cor}[Corrolary 2.3 in \cite{janson}]
Let $X \sim Bin(n,p)$ be a binomial variable. Then for all $\eps \leq \frac{3}{2}$,
\begin{equation}
\label{eq:bin_exp_scale_tail}
\Prob{|X - \Exp X| \geq \eps \cdot \Exp X  } \leq 2 \exponent{ - \frac{\eps^2}{3}  \Exp X}
\end{equation}
\end{cor}

We will also often use the following Corollary of Theorem 
\ref{thm:chernoff_main}. 
\begin{cor}
\label{cor:my_best_chernoff}
There is a constant $c>0$ such that the following holds: \\
Let $X \sim Bin(n,p)$ be a binomial variable such that $\lambda = np  > 1$. Then for any $N>0$, 
\begin{equation}
\label{eq:binomial_bound_var}
\Prob{  | X - \Exp X| \geq 20 \cdot \sqrt{\lambda} \cdot \log N } \leq 
  c / N^2.
\end{equation}
\end{cor}

\newcommand{\GPQ}{\rgraph{G}_{p,q}}
We now present a series of Lemmas about random graphs in the $p,q$-
SBM model and about random initializations. Throughout $G=(V,E)$ will 
be assumed to be a random graph from the $p,q$-SBM and we denote this 
$G \sim \GPQ$.  Recall that $N = |V|$ is the size of the node set, and for a node $i \in V$ in a fixed 
graph $G$, $n_i$ is the set of neighbours of $i$, and $d_i = |n_i|$ is 
the degree of $i$. Also, for a set $S \subset V$, its full degree is $d_S = \sum_{i \in S} d_i$. Next, for a set $S \subset V$, we denote by $d(i,S) = |n_i 
\cap S|$ the number of edges between $i$ and $S$ and for two 
sets, $S,T \subset V$ define $d(S,T) = \sum_{i \in S} d(i,T)$ to be the number of edges between $S$ and $T$. Finally, 
set $d_2(i,T) = d(n_i,T)$ to be the number of paths of length two 
that start at $i$ and end at $T$. 

In addition, let $C_1,C_2$, with $|C_1|=|C_2|=N/2$, be a random partition of $V$ into two sets, the initialization of DER.  
Denote $N_1 = |C_1 \cap P_1|$, and $N_2 = N/2 - N_1 = |C_1 \cap P_2| = |C_2 \cap P_1|$. We assume without loss of generality that $N_1 \geq N_2$, and set $\Delta N = N_1 - N_2$. The partition $C_1,C_2$ will be considered fixed in all the lemmas that concern the random graphs.

We proceed to give bounds on the expectations and concentration intervals of several quantities related to our problem. 

For a fixed node $i \in V$,  the degree $d_i$ is distributed as
a sum of two independent binomials, 
\begin{equation}
\label{eq:node_degree_decomposition}
 d_i \sim Bin(N/2 -1 ,p) + Bin(N/2,q), 
\end{equation}
the first term counts the edges to the component to which $i$ belongs, the second to the other component. In particular,
the expected degree is 
\begin{equation}
\label{eq:node_degree_expectation}
 \Exp d_i = (N/2 -1)p + (N/2)q.  
\end{equation}

\begin{lem}[Degree bounds]
\label{lem:total_deg_bound}
Let $G \sim \GPQ$.  There exists a constant $\hat{c}_1$ such that the following holds:
 Assume that 
\begin{equation}
 \label{eq:deg_nbd_lem_assum}
 Np \geq 100 \log N. 
\end{equation}
 Then with probability at least $1 - \hat{c}_1/N$,  for all $i \in V$ 
\begin{equation}
\label{eq:full_node_deg_bound}
  \quater \cdot \frac{N}{2} p \leq d_i \leq 2 \cdot N p.
\end{equation}
\end{lem}
\begin{proof} 
Fixed a node $i \in V$, and let $X \sim Bin(N/2 -1 ,p)$ and  $Y \sim 
Bin(N/2,q)$ be two independent binomials such that $d_i \sim X+Y$. 
By applying (\ref{eq:bin_exp_scale_tail}) to $X$ with $\eps = \half$, 
we obtain that 
\begin{equation}
 \quater \frac{N}{2} p  \leq \Exp X - \half \Exp X \leq X < d_i 
\end{equation}
with probability at least $1 - 2exp(-\frac{1}{12} (\frac{N}{2}p -1))$.  Using the assumption (\ref{eq:deg_nbd_lem_assum}), it follows that there is $c>0$ such that $2exp(-\frac{1}{12} (\frac{N}{2}p -1)) \leq c/N^2$. Using the union bound we therefore conclude that 
\begin{equation}
 \quater \frac{N}{2} p  \leq d_i 
\end{equation}
holds for all nodes $i \in V$ with probability at least $1 - c/N$. 
Similarly, we use (\ref{eq:bin_exp_scale_tail}) to obtain 
that 
$X \leq Np$ with probability at least $1 - c/N^2$, perhaps with a different $c$ and that $Y \leq Np$ with probability at least $1 - c'/N^2$, because $q < p$. By the union bound it follows that $d_i = X+Y \leq 2 Np$ with probability at least $1 - (c+c')/N^2$, and by the union bound again, we obtain $d_i \leq 2Np$ for all $i\in V$, with probability ate least $1 - c'' / N$. 
\end{proof}

In what follows we will often encounter situations where we need to 
bound fluctuations of sums of a fixed number of not necessarily 
independent random variables, and considerations similar to those in 
Lemma \ref{lem:total_deg_bound} will often be omitted.

We now consider the degree of $C_1$, $d_{C_1}$. Note that by symmetry $\Exp d_{C_1} = \Exp d_{C_2}$, and that the total degree of the graph satisfies $d_G = d_{C_1} + d_{C_2}$. Therefore 
\begin{equation}
\label{eq:exp_dc1}
\Exp d_{C_1} = \half d_G = N \Exp d_i = N \left((N/2 -1)p + (N/2)q\right).
\end{equation}

The next lemma concerns the concentration of the degree of $C_1$. 

\begin{lem}
  Set $\lambda = N^2 p$.
  There exist constants $\hat{c_3}$,$\hat{c_4}$ such that with probability at least $1 - \hat{c_3} / N$, 
\begin{equation}
\label{eq:exp_dc1_var}
   | d_{C_1} - \Exp d_{C_1}| \leq \hat{c_4} \log N \cdot \sqrt{\lambda}.
\end{equation}  
\end{lem}
\begin{proof}
For $l,s \in \{1,2\}$, set $S_{ls} = C_l \cap P_s$. Observe that 
$d_{C_1}$ can be written as 
\begin{eqnarray*}
d_{C_1} = 2 \cdot d(S_{11},S_{11}) + 2 \cdot d(S_{12},S_{12}) +
  2 \cdot d(S_{11},S_{12}) + \\
  + d(S_{11},S_{21}) +  d(S_{11},S_{22}) + \\
  +  d(S_{12},S_{21}) +  d(S_{12},S_{22}). 
\end{eqnarray*}
Note that each of the terms in the sum above is a binomial variable with lambda that is smaller or equal to $c N^2 p$ for some constant $c >0$. Therefore by applying Corollary \ref{cor:my_best_chernoff} to each term and using union bound, we obtain the result. 
\end{proof}

The next Lemma provides an upper bound on $\Delta N$. 

\begin{lem}
\label{lem:delta_n_bound}
There are constants $c_1,c_2>0$ such that 
\begin{equation}
\label{eq:delta_n_bound}   
\Delta N \leq c_1 \sqrt{N} \log N
\end{equation}
with probability at least $1 - c_2 / N$. 
\end{lem}
\begin{proof}
For the purposes of this lemma we do not assume that $N_1 > N_2$. 
Recall that $N_1$ is the size of the intersection $P_1$ with 
a random subset of $V$ of size $N/2$, denoted $C_1$.  
Hence $N_1$ has has the hypergeometric distribution. 
Set 
\begin{equation}
\label{eq:n1_exp}
  \lambda = \Exp N_1 = \frac{|P_1||C_1|}{|V|} = \quater N. 
\end{equation}
The hypergeometric distribution satisfies concentration inequalities similar to those satisfied by the binomials. Specifically, by Theorem 2.10 in \cite{janson}, the conclusion of Corollary \ref{cor:my_best_chernoff}, inequality (\ref{eq:binomial_bound_var}) holds for hypergeometric variables, with $\lambda$ is defined as in (\ref{eq:n1_exp}). The result follows by an application of that inequality. 
\end{proof}

We now examine the quantity $d(j,C_2)$ for a node $j \in V$. 
The expectations satisfy 
\begin{eqnarray}
\Exp d(j,C_2) = N_2 p + N_1 q \mbox{ \spaceo if $j \in P_1$} \label{eq:djc2_exp_1}\\
\Exp d(j,C_2) = N_1 p + N_2 q \mbox{ \spaceo if $j \in P_2$}. \label{eq:djc2_exp_2}
\end{eqnarray}
This follows from the decomposition of $d(j,C_2)$ as a sum of two binomials. Similar expressions hold also for $d(j,C_1)$. 
Note that when, for instance $j \in P_1$,  in fact $\Exp d(j,C_2) = N_2 p + N_1 q$ if $j \in C_1 \cap P_1$, and $\Exp d(j,C_2) = (N_2-1) p + N_1 q$ if $j \in C_1 \cap 
P_1$. Since we will be interested only in orders of magnitude, we 
will disregard the difference between the two cases in what follows. Throughout the proof we denote 
\begin{equation}
\label{eq:L}
 L = N_2 p + N_1 q
\end{equation}
as a convenient shorthand for $\Exp d(j,C_2)$ (when  $j \in P_1$). 

The quantities in the following Lemma will be relevant in what follows:
\begin{lem} 
\label{lem:djc}
Assume that the partition $C_1,C_2$ is such that 
\begin{equation}
\label{eq:djc_n_assumption}
\Delta N \leq c \sqrt{N} \log N.
\end{equation} 
Then there exist constants 
$c_1,c_2,c_3,c_4 > 0$ and $\kappa_1 > 0$ such that if $Np > \kappa_1$ then with probability at least $1- \frac{c_1}{N}$ the following holds:
For all $j \in V$, 
\begin{eqnarray}
d(j,C_2) \geq c_2 N p \label{eq:lem_djc_claim1}\\
|d(j,C_1) - d(j,C_2)| \leq c_3 \sqrt{Np} \log N \label{eq:lem_djc_claim2}\\
d(j,C_1) / d(j,C_2) \geq \half \label{eq:lem_djc_claim3} \\
\Abs{d(j,C_2) - L} \leq c_4 \sqrt{Np} \log N. \label{eq:lem_djc_claim4}
\end{eqnarray}
\end{lem}
\begin{proof}
We show that the statements hold for every $j \in V$ individually with probability at least $1 - c_4/ N^2$, from which the claim of the Lemma follows by the union bound. 

Using inequality (\ref{eq:binomial_bound_var}), 
we obtain that with probability at least $1 - c_5/N^2$,
\begin{equation}
\label{eq:lem_djc_djc2_dev}
|d(j,C_2) - \Exp d(j,C_2)| \leq c_6 \sqrt{Np} \log N,
\end{equation}
and similarly
\begin{equation}
\label{eq:lem_djc_djc1_dev}
|d(j,C_1) - \Exp d(j,C_1)| \leq c_6 \sqrt{Np} \log N, 
\end{equation}
where in a way similar to the proof of Lemma 
\ref{lem:total_deg_bound}, we have used the decomposition of 
$d(j,C_l)$ into two binomials and the fact that $q<p$. 

Assume that $Np$ is large enough so that 
\begin{equation}
\label{eq:lem_djc_c2_var2exp_bound}
c_6 \sqrt{Np} \log N \leq \frac{1}{10} Np 
\end{equation}
holds. 

By using the assumption (\ref{eq:djc_n_assumption}) and 
(\ref{eq:djc2_exp_1}) or (\ref{eq:djc2_exp_2}), 
we obtain that 
\[
	\Exp d(j,C_2) \geq \quater Np  
\]
for all $N \geq \kappa_2$ for some constant $\kappa_2 > 0$. Combining this with (\ref{eq:lem_djc_djc2_dev}) and with (\ref{eq:lem_djc_c2_var2exp_bound}), we obtain 
\begin{equation}
d(j,C_2) \geq \Exp d(j,C_2) -  c_6 \sqrt{Np} \log N \geq (\quater - \frac{1}{10}) Np, 
\end{equation}
thereby proving (\ref{eq:lem_djc_claim1}).  Next, using
(\ref{eq:djc2_exp_1}), (\ref{eq:djc2_exp_2}) and similar expressions for $d(j,C_1)$ we obtain that 
\begin{equation} 
\label{eq:lem_djc_diff_exp}
|\Exp d(j,C_1) - \Exp d(j,C_2)|  =  \Delta N (p-q).
\end{equation}
Using (\ref{eq:lem_djc_diff_exp}) with (\ref{eq:lem_djc_djc2_dev}) and (\ref{eq:lem_djc_djc1_dev}), it follows that 
\begin{equation}
\label{eq:lem_djc_diff_bound}
|d(j,C_1) - d(j,C_2)| \leq c \Delta N p  + c' \sqrt{Np} \log N \leq c_8 \sqrt{Np} \log N, 
\end{equation}
for appropriate constants $c,c' > 0$. This proves  (\ref{eq:lem_djc_claim2}. Similarly, 
the claim (\ref{eq:lem_djc_claim4}) holds for all $j \in P_1$ and for $j \in P_2$ we have
\begin{eqnarray*}
\Abs{d(j,C_2) - L} & \leq & \Abs{L - \Exp d(j,C_2)} + c'' \sqrt{Np} \log N \leq \\
& \leq & c \Delta N p  + c'' \sqrt{Np} \log N \leq c_9 \sqrt{Np} \log N.
\end{eqnarray*}
Thus (\ref{eq:lem_djc_claim4}) holds for all $j \in V$. Finally, to show (\ref{eq:lem_djc_claim3}) 
write 
\begin{equation}
\frac{d(j,C_1)}{d(j,C_2)} = 1 - \frac{d(j,C_1) - d(j,C_2)}{d(j,C_2)}.
\end{equation}
Then (\ref{eq:lem_djc_claim3}) holds if $|\frac{d(j,C_1) - d(j,C_2)}{d(j,C_2)}| \leq \half$ holds, which in turn holds by (\ref{eq:lem_djc_claim1}) and (\ref{eq:lem_djc_claim2}) , for $N$ and $Np$ larger than some fixed constant. 
\end{proof}

We now provide some estimates on the number of length two paths (which we also referr to as 2-paths in what follows). 
\begin{lem}  For a node $j \in P_1$, 
\label{lem:cru_lem_1}
\begin{eqnarray}
\label{eq:d2_paths_exp_1}
\Exp d_2(j,C_1) = \half N \left( N_1 p^2 + 2pq N_2 + N_1 q^2\right) \\
\label{eq:d2_paths_exp_2}
\Exp d_2(j,C_2) = \half N \left( N_2 p^2 + 2pq N_1 + N_2 q^2\right)
\end{eqnarray}
\end{lem}
\begin{proof}
For $l,s \in \{1,2\}$, set $S_{ls} = C_l \cap P_s$. There are four types of 2-paths from $j$ to $C_1$. Those that land in $P_1$ at first step, and then land at $S_{11}$. We denote paths of this type by $P_1S_{11}$. There exist $\half N \cdot N_1$ such possible paths and 
each one exists in $\GPQ$ model with probability $p^2$. For some concrete path of type $P_1S_{11}$, say $p = j,u,v$, with $u \in P_1$ and $v \in S_{11}$, let $E_p$ be the event that this path exists in the graph. The number of such paths is then $\sum_{p \in P_1 S_{11}} \bold{1}_{E_p}$ and the expected number of such paths is therefore $\half N N_1 p^2$.  The other path types are $P_1S_{12}$,  $P_2S_{11}$,$P_2S_{12}$, with expected numbers of paths $\half N N_2 pq$,$\half N N_1 q^2$ and $\half N N_2 pq$ respectively. Hence (\ref{eq:d2_paths_exp_1}) holds. Similar considerations yield (\ref{eq:d2_paths_exp_2}).
\end{proof}

Next we obtain concentration bounds on $d_2$. 
\begin{lem} 
\label{lem:cru_lem_2}
There are constants $c_1,c_2 > 0$, such that with probability 
at least $1 - c_1/N$ the following holds: For all $i \in P_1$, 
\begin{eqnarray}
|d_2(i,C_1) - \Exp d_2(i,C_1)| \leq c_1 Np \log N \\
|d_2(i,C_2) - \Exp d_2(i,C_2)| \leq c_1 Np \log N 
\end{eqnarray}
\end{lem}
\begin{proof}
Let $n_i$ be the neighbourhood of $i$ in $G$. Set as before 
$S_{ls} = C_l \cap P_s$ for $l,s \in \{1,2\}$ and set also $A_{ls}= S_{ls} \cap n_i$. 
Similarly to the arguments in the previous Lemmas, to obtain concentration bounds 
on $d_2(i,C_1)$ we represent it as a sum of binomials 
\[
	d_2(i,C_1) = \sum_{l,s \in \{1,2\}} \sum_{t,r \in \{1,2\}} d(A_{ls},S_{tr}). 
\]
Then one observes that the lambda of each such binomial is of the 
order $Np \cdot N \cdot p$, because the size of $A_{ls}$ is of the 
order of $Np$ and the size of $S_{tr}$ is of the order of $N$. Then 
the conclusion follows by inequality (\ref{eq:binomial_bound_var}). 
Since the sets $A_{ls}$ are random sets, to carry the above argument 
precisely we first condition on the neighbourhood of $n_i$ and ensure 
(using (\ref{eq:bin_exp_scale_tail})) that the sets $A_{ls}$ are 
indeed not larger that $cNp$ for an appropriate $c>0$. The full 
details are straightforward but somewhat lengthy and are omitted.
\end{proof}

We will also make use of the following inequalities:
\begin{eqnarray}
	 \log (1 + t) \leq t \mbox{ for all $t \geq -1$} \label{eq:log_ineq_upper} \\
	t - t^2 \leq \log (1 + t) \mbox{ for all $t \geq -\half$} \label{eq:log_ineq_lower}\\
	| \log 	\frac{t}{s} | \leq \frac{|t-s|}{\min\{t,s\}} \mbox{ for all $t,s > 0$} \label{eq:log_abs} \\
	|\frac{s}{t+\theta} - \frac{s}{t}| = 
	|\frac{\theta}{t+\theta}| \cdot |\frac{s}{t} |
	\mbox{ for all $t,s,\theta$} \label{eq:frac_diff} 
\end{eqnarray}

\begin{proof}[Proof of Theorem \ref{thm:pq_bound}:]
For $x \in V$, denote 
by $n_x$ the set of neighbours of $x$ in $G$. As indicated earlier, we shall use that fact that $C_1$ 
is slightly biased towards either $P_1$ or $P_2$. Specifically, set $\delta = \half \epsilon$ and assume throughout the proof, without loss of generality, that $N_1 > 
N_2$. Then the following holds with high probability:
\begin{equation}
\label{eq:ndiff}
	\Delta N = N_1 - N_2 \geq N^{\half - \delta}. 
\end{equation}
Indeed, note that $N_1$, as a function of the random partition, is hypergeometrically distributed with mean $N/4$ and standard deviation of order $N^{\half}$. Hence, by the central limit theorem 
for the hypergeometric distribution (see ~\cite{feller};~\cite{nicholson}), 
\begin{equation}
\label{eq:clt}
\Prob{ \left| N_1 - \quater \cdot N \right| \geq  N^{\half - \delta}} \rightarrow 1
\end{equation}
with $N \rightarrow \infty$. Statement (\ref{eq:clt}) guarantees a deviation from the mean, and 
in particular that (\ref{eq:ndiff}) holds with high probability.

To prove the Theorem we now establish the following claim:
\begin{claim}
\label{claim:domination}
Fix a partition $C_1,C_2$ of $V$, satisfying eq. (\ref{eq:ndiff}) and (\ref{eq:djc_n_assumption}). 
Under assumptions (\ref{eq:p_size}) and (\ref{eq:q_size}),
with probability at least $1 - \frac{1}{N}$ graph $G$ satisfies: For all $i \in P_1$, 
\begin{equation}
\label{eq:d_win}
	D(w_i,\mu_{C_1}) > D(w_i,\mu_{C_2}). 
\end{equation}
\end{claim}

Note that the assumptions of the Claim depend only on randomness of the partitions and are satisfied with high probability. Indeed, (\ref{eq:ndiff})  holds as discussed above and (\ref{eq:djc_n_assumption}) follows from Lemma (\ref{lem:delta_n_bound}). 

Once we prove the claim, by symmetry we will also have for all $i \in P_2$ a reverse inequality in 
(\ref{eq:d_win}), and together with (\ref{eq:ndiff}) this will prove the theorem.   We proceed to prove the claim. 

 Observe that by  definition we have $\mu_{C_l}(i) = \frac{d(i,C_l)}{d_{C_l}} $ for every $i \in V$. 

Therefore we can rewrite (\ref{eq:d_win}) as:
\begin{eqnarray}
\label{eq:simpl_step_1}
&\sum_{j \in n_i} \log \frac{d(j,C_1)}{d(j,C_2)} + \\ 
&+ d_i \log \frac{d_{C_2}}{d_{C_1}}  \label{eq:c_degree_term}\\
& >  0
\end{eqnarray}

We now bound the term (\ref{eq:c_degree_term}). Using (\ref{eq:log_abs}) we obtain 
\begin{equation}
\left| \log \frac{d_{C_2}}{d_{C_1}}  \right| \leq 
\frac{|d_{C_2} -d_{C_1}|}{\min\{d_{C_2},d_{C_1}\}}.
\end{equation}
Using (\ref{eq:exp_dc1}) and (\ref{eq:exp_dc1_var}) we obtain that 
\begin{equation}
\min\{d_{C_2},d_{C_1}\} \geq cN^2 p , 
\end{equation}
and that 
\begin{equation}
|d_{C_2} -d_{C_1}| \leq c N \log N \sqrt{p}. 
\end{equation}
In addition, recall that by Lemma \ref{lem:total_deg_bound}, $d_i \leq c Np$. Therefore we obtain that 
\begin{equation}
\label{eq:big_degree_bound}
\left| d_i \log \frac{d_{C_2}}{d_{C_1}} \right| \leq c Np \frac{N \log N \sqrt{p}}{c'' N^2p } \leq c''' \log N \sqrt{p} \leq c'''\log N
\end{equation}
for some constant $c'''>0$. 

We now examine the term (\ref{eq:simpl_step_1}). 
Using (\ref{eq:log_ineq_lower}), write 
\begin{equation}
\label{eq:first_term_open}
\log \frac{d(j,C_1)}{d(j,C_2)} \geq \frac{d(j,C_1) - d(j,C_2)}{d(j,C_2)} - \left(\frac{d(j,C_1) - d(j,C_2)}{d(j,C_2)} \right)^2.
\end{equation}
Note that by (\ref{eq:lem_djc_claim3}), $\frac{d(j,C_1)}{d(j,C_2)} \geq \half$ and therefore (\ref{eq:log_ineq_lower}) applies.  We now replace the denominator in the first term of the right hand of (\ref{eq:first_term_open}) by a quantity independent of $j$, namely by $L$ as defined in (\ref{eq:L}). Using (\ref{eq:frac_diff}) with 
$s = d(j,C_1) - d(j,C_2)$, $t = L$ and $\theta = d(j,C_2) - L$, write 
\begin{equation}
\frac{d(j,C_1) - d(j,C_2)}{d(j,C_2)} \geq \frac{d(j,C_1) - d(j,C_2)}{L} - 
         \frac{\Abs{d(j,C_2) - L}}{d(j,C_2)} \cdot  \frac{\Abs{d(j,C_1) - d(j,C_2)}}{L}.
\end{equation}
To summarize, we have obtained that 
\begin{eqnarray}
& \sum_{j \in n_i} \log \frac{d(j,C_1)}{d(j,C_2)} & \geq  \label{eq:long_lin_split1}\\
&  \sum_{j \in n_i} \frac{d(j,C_1) - d(j,C_2)}{L} &   \label{eq:long_lin_split2}\\
& - \sum_{j \in n_i} \frac{\Abs{d(j,C_2) - L}}{d(j,C_2)} \cdot  \frac{\Abs{d(j,C_1) - d(j,C_2)}}{L} &  \label{eq:long_lin_split3}\\
& -\sum_{j \in n_i} \left(\frac{d(j,C_1) - d(j,C_2)}{d(j,C_2)} \right)^2. & \label{eq:long_lin_split4}
\end{eqnarray}
Note that the term (\ref{eq:long_lin_split2}) satisfies 
\begin{equation}
\label{eq:heart_of_proof}
\sum_{j \in n_i} \frac{d(j,C_1) - d(j,C_2)}{L} = \frac{d_2(i,C_1) - d_2(i,C_2)}{L}.
\end{equation}
This term counts the number of 2-paths and is the heart of the proof. Before analysing it, we bound the other 
two terms in the inequality in (\ref{eq:long_lin_split1}). Plugging in the estimates from Lemma \ref{lem:djc}, we obtain for (\ref{eq:long_lin_split3}) that 
\begin{equation}
\sum_{j \in n_i} \frac{\Abs{d(j,C_2) - L}}{d(j,C_2)} \cdot  \frac{\Abs{d(j,C_1) - d(j,C_2)}}{L} \leq 
c \cdot d_i \frac{\sqrt{Np}\log N}{Np} \cdot \frac{\sqrt{Np}\log N}{Np}.
\end{equation}
Using the degree estimate form Lemma \ref{lem:total_deg_bound}, $d_i \leq c Np$, we thus get 
\begin{equation}
\label{eq:linearization_diff_term}
\sum_{j \in n_i} \frac{\Abs{d(j,C_2) - L}}{d(j,C_2)} \cdot  \frac{\Abs{d(j,C_1) - d(j,C_2)}}{L} \leq 
c (\log N )^2
\end{equation}
for an appropriate $c>0$. Similarly, for the term (\ref{eq:long_lin_split4}) we have 
\begin{equation}
\label{eq:linearization_quadratic_term}
\sum_{j \in n_i} \left(\frac{d(j,C_1) - d(j,C_2)}{d(j,C_2)} \right)^2 \leq c \cdot d_i \cdot
\frac{Np \log^2 N}{ N^2p^2} \leq c \cdot \log^2 N,
\end{equation}
with some (perhaps different) $c >0$. 

We now proceed to obtain a lower bound on (\ref{eq:heart_of_proof}). 
The crucial property of length two path counts, $d_2(i,C_1)$ and 
$d_2(i,C_2)$, that enables such a bound is that the difference 
between the expectations of these quantities is of larger order of 
magnitude than their fluctuations. 

Indeed, by Lemma \ref{lem:cru_lem_2}, with probability at least 
$1 - c/N$ we have that 
\begin{equation}
d_2(i,C_1) - d_2(i,C_2) \geq \Exp d_2(i,C_1) - \Exp d_2(i,C_2) - 2 c Np \log N
\end{equation}
for all $i \in P_1$. In addition, by Lemma \ref{lem:cru_lem_1}, 
\begin{equation}
\Exp d_2(i,C_1) - \Exp d_2(i,C_2) = \half N \Delta N (p-q)^2 \geq 
  N^{3/2 - \delta} (p-q)^2 ,
\end{equation}
where we have used (\ref{eq:ndiff}) in the last inequality. 

Incorporating the inequalities (\ref{eq:big_degree_bound}), (\ref{eq:linearization_diff_term}), 
(\ref{eq:linearization_quadratic_term}), we obtain that $D(w_i,\mu_{C_1}) > D(w_i,\mu_{C_2})$ holds if the 
following inequality holds:
\begin{equation}
\label{eq:final_thm_ineq}
\frac{N^{3/2 - \delta} (p-q)^2 - 2 c Np \log N}{L} - c \log N > 0.
\end{equation}

To prove the theorem, it remains to choose $p$ and $q$ such that (\ref{eq:final_thm_ineq}) is satisfied. Such 
$p,q$ are given by the assumptions (\ref{eq:p_size}), (\ref{eq:q_size}). Indeed, recall that $L$ satisfies $L \leq c Np$ for an appropriate $c>0$ and hence under assumptions (\ref{eq:p_size}), (\ref{eq:q_size}) we have
\begin{equation}
\frac{N^{3/2 - \delta} (p-q)^2 - 2 c Np \log N}{L \log N}  \rightarrow \infty
\end{equation}
with $N \rightarrow \infty$, hence yielding (\ref{eq:final_thm_ineq}).
\end{proof}


\section{LFR benchmarks}
\label{sec:lfr_app_stuff}

In this section we specify the full parameters used for the experiments in the paper. 

The LFR model is generated from the following parameters: The graph size $N$, the mixing parameter $\mu$, 
community size lower and upper bounds $c_{min},c_{max}$, average degree $d$, maximal degree $d_{max}$, and the 
power law exponents for the degree and community size distributions - which are in all cases set to their 
default values of $-2$ and $-1$ respectively.  In addition, in the overlapping case, parameter $n$ specifies 
the number of nodes that will participate in multiple communities, and the parameter $m$ specifies the number 
of communities in which each such node will participate. 

The LFR models were generated using the software available at \cite{lfr_gen_site}.

For the non overlapping LFR benchmarks we have used $d = 20$ and $d_{max} = 50$, with the rest of parameters 
as specified in Section 4.2. This corresponds precisely to the experiments in \cite{LFComp}.
The repeats strategy is described in Section \ref{sec:repeats_and_restarts}.  For each given graph instance, 
DER was run with 15 repeats, using 3 restarts in each run. The results of the repeats were clustered using the 
threshold algorithm described in Section 
\ref{sec:repeats_and_restarts}, except in the $\mu = 0.7$ in which we have used the spectral clustering to 
cluster the co-occurence matrix. 

The LFR experiments with the spectral clustering algorithm that are shown in Figure 2.b 
were performed using the spectral clustering version in Python sklearn v0.14.1 package, which is an 
implementation of the algorithm in \cite{NJW}. The spectral clustering was run with 150 restarts of its 
final stage Euclidean k-means step. We note that while the repeats strategy could be applied to the spectral 
clustering too, it did not improve the performance in this case (despite the fact that different runs of 
spectral clustering returned somewhat different results). The results shown in Figure 2.b 
are without repeats.

For the overlapping community benchmarks we have used the following settings: $N=10000$, $d=60$, $d_{max} = 
100$, $c_{min} = 200$, $c_{max} = 500$. The value of $n$ was $5000$ and $m$ was $4$. These are the settings 
that were used in \cite{GopBlei}. As discussed in the next section, in one sense these settings can be 
considered a heavy overlap, while there is a different sense in which they can be considered sparse.  In all 
cases we have run DER with 15 repeats and 3 restarts per run, and we have used the spectral clustering to 
cluster the co-occurence matrix. 

Recall that our approach to overlapping communities is to first obtain a non-overlapping clustering and then 
to post-process it to obtain overlapping communities. One can ask therefore what will happen if in the non-
overlapping step, DER is replaced by another non-overlapping clustering algorithm. We have tried using 
spectral clustering instead of DER, and applied the same post-processing. In all cases this resulted in ENMI 
values close to $0$. 

\section{Overlapping LFR benchmarks}
\label{sec:lfr_app_overlap_stuff}

We refer to \cite{LFBench} and \cite{LFRBench} for the definitions of the LFR models. In this section we make 
a few brief comments regarding the structure of the overlapping LFR communities. 

To simplify the discussion, we restrict our attention to the particular settings that were used in the 
evaluation in Section 4.3. The settings $n=5000$ and $m=4$ (see Section 
\ref{sec:lfr_app_stuff})  
imply that there are $5000$ such that each node belongs to a single community, and $5000$ nodes such that each 
node belongs to $4$ communities. These settings may be considered as a heavy overlap (see \cite{XieSurvey}). 
Indeed, it follows theoretically from the way LFR communities are generated, and also is observed in actual 
graphs, that under these settings each community $C$ contains about $20\%$ of nodes that belong only to $C$, 
and each of the remaining $80\%$ of the nodes belongs to $C$ and to 3 other communities. 

On the other hand, for a node $i \in C$ that belongs to 3 other communities, the 3 other communities are 
chosen at random among about $75$ remaining communities of the graph. This implies that for each pair of 
communities $C,J$, the intersection between them is small and if a node $i \in V$ is chosen at random, 
the event $i \in C$ is almost independent of the event $i \in J$. 

The above small intersections and lack of correlations between communities property implies that random walk 
started from community $C$, after two steps has a chance of about $1/16$ of returning to $C$ while the rest 
of the probability is distributed more or less uniformly between the other communities (and is much less than 
$1/16$ for each community that is not $C$). In other words, the measures $w_i$ and $w_j$ have much more chance 
of being correlated if $i$ and $j$ belong to some common $C$ than otherwise. This explains why DER works well 
on these graphs.


\end{document}